\documentclass[letterpaper, 10 pt, conference]{ieeeconf}
\usepackage{cite}
\usepackage[pdftex]{graphicx}
\usepackage[cmex10]{amsmath}
\usepackage{amssymb,graphicx,charter, latexsym}
\usepackage{amsfonts}
\usepackage{bbm}
\usepackage{amsfonts}
\usepackage{bbm}
\usepackage{algorithm}
\usepackage{algorithmic}
\usepackage{tikz} 
\usepackage{amsfonts}
\usepackage{verbatim}
\usepackage{amssymb}
\newtheorem{lemma}{Lemma}
\usepackage{comment}

\newtheorem{theorem}{Theorem}
\newtheorem{condition}{Condition}

\IEEEoverridecommandlockouts                              % This command is only
\begin{document}
\title{Kiefer-Wolfowitz Algorithm is Asymptotically Efficient for a Class of Non-Stationary Bandit Problems
}
\author{ Rahul Singh and Taposh Banerjee% <-this % stops a space
\thanks{Rahul Singh is with the Laboratory of Information and Decision Systems (LIDS), Massachusetts Institute of Technology, Cambridge, MA 02139, USA; Taposh Banerjee is with SEAS, Harvard University, Cambridge, MA. 
{\tt\small rsingh12@mit.edu, tbanerjee@seas.harvard.edu.}
 } %
 }
\maketitle
\begin{abstract}
We consider the problem of designing an allocation rule or an ``online learning algorithm" for a class of bandit problems in which the set of control actions available at each time $s$ is a convex, compact subset of $\mathbb{R}^d$. Upon choosing an action $x$ at time $s$, the algorithm obtains a noisy value of the unknown and time-varying function $f_s$ evaluated at $x$. The ``regret" of an algorithm is the gap between its expected reward, and the reward earned by a strategy which has the knowledge of the function $f_s$ at each time $s$ and hence chooses the action $x_s$ that maximizes $f_s$.
%(a popular choice is the Kiefer Wolfowitz (KW) algorithm~\cite{kiefer1952stochastic})
For this non-stationary bandit problem set-up, we consider two variants of the Kiefer Wolfowitz (KW) algorithm i) KW with fixed step-size $\beta$, and ii) KW with sliding window of length $L$. We show that if the number of times that the function $f_s$ varies during time $T$ is $o(T)$, and if the learning rates of the proposed algorithms are chosen ``optimally", then the regret of the proposed algorithms is $o(T)$, and hence the algorithms are asymptotically efficient. %Furthermore, when the number of variations in $f_t$ is $O(T)$, the algorithms yield a regret  
\end{abstract}
\section{Introduction}\label{sec:intro}
The Multi-Armed Bandit problem (MABP) requires a player to play an arm at each time $s=1,2,\ldots$ from a set of arms. If $X_s$ denotes the arm played at time $s$, then the player receives a random reward at time $s$, the distribution of which depends on $X_s$. The objective of the player is to maximize the \emph{expected value} of the cumulative reward collected over a period of time $T$. The player does not know the mean value of the random reward as a function of the choice $x$, and hence the control action corresponding to the choice of arm to be played needs to balance an exploration-exploitation trade-off.

This paper is concerned with a particular class of bandit problems in which the control action available to the player can be mapped to a convex compact subset of $\mathbb{R}^d$, i.e., the continnum bandit problem~\cite{agrawal1995continuum}, in which the mean reward of the arms is non-stationary. The addition of non-stationarity into the MABP adds to the complexity involved in the exploration-exploitation dilemma, since now the player's belief about the mean reward of an arm cannot depend upon past data that is ``too old" because the reward distribution of arms might have changed since the time that information was collected. Thus, the learning rate of the player has to be suitably adapted to the rate-of-change of the mean reward function.

% Thus, he either needs to discount the past information based on its ``age", or discard it completely if its age is above a certain threshold. The choice of how 

%This class of bandit problems is called the continnum Thus, playing an arm needs to balance  exploration of the mean reward function, while simultaneoulsy ``exploiting" arms having  

%distribution of the arms by playing each arm sufficiently many times and ``learning" the mean value of the reward obtained from it. Meanwhile, the player also needs to play arms which have a 

\section{Kiefer Wolfowitz algorithm}
Let $\mathcal{D}$ be a compact and convex subset of $\mathbb{R}^d$. The original KW algorithm was designed in the context of maximizing a fixed function by obtaining noisy samples of the function values. We begin by describing the KW algorithm for the case when the function $f:\mathcal{D}\to\mathbb{R}$ to be optimized is fixed. The maximizer of $f$ is denoted $\theta(f)\in \mathcal{D}$. The vanilla version of the KW algorithm maintains, at each time-step $s$ an estimate of the function maximizer, denoted as $X_s=\left(X_s(1),X_s(2),\ldots,X_s(d)\right)$. It then makes an estimate of the derivatives $\left(\nabla f\right)_{X_s}(i)$ of the unknown function $f$ by sampling the function values at points $X_s+ c_s e(i),i=1,2,\ldots,d$ and $X_s-c_se(i),i=1,2,\ldots,d$, where $e(i)$ is the unit vector with $1$ in the $i$-th place. 
 Let $F^+_{s}(i),F^-_{s}(i)$ be the noisy values of the function at $X_s+c_se(i)$ and $X_s-c_se(i)$ respectively.  Denote by $Y_s$ the estimated value of the derivative of function $f$ at $X_s$. If $Y_s=\left(Y_s(1),Y_s(2),\ldots,Y_s(d)\right)$ is an estimate of $\nabla f$ at $X_s$, we then have that,
\begin{align}\label{eq:1}
Y_s(i) = \frac{F^{+}_s(i)-F^{-}_s(i)}{2c_s}, i=1,2,\ldots,d,
\end{align} 
where $Y_s(i)$ is an estimate of $\left(\nabla f\right)(i)$, i.e., the $i$-th component of the gradient  at $X_s$.
Once an estimate of the derivative of $f$ at $X_s$ has been made, the KW algorithm then updates the estimate of maximizer as follows,
\begin{align}\label{eq:2}
X_{s+1} = X_s + \beta_s Y_s,
\end{align}
where $\beta_s$ is called the learning rate. Typically the step sizes are chosen as $\beta_s = s^{-1/2}, c_s = s^{-1/4}$. A detailed description of the KW algorithm can be found in~\cite{kushner}.
%For the case when the dimension  $d$ of the vector space satisfies $d>1$, then a somewhat related algorithm, called the Simultaneous Perturbation Stochastic Approximation (SPSA)~\cite{spall1992multivariate} can also be employed to maximize the unknown $f$. SPSA has an advantage over KW in that it utilizes a single function measurement in order to update the gradients w.r.t. each of the $d$ dimensions. %We will exclusively cover the case when $d=1$, and believe that the analysis can be extended easily to cover the case of $d>1$.
\section{Past Works and Contributions}
A survey of the results on MABP literature can be found in~\cite{bubeck2012regret}.~\cite{agrawal1995continuum} is the first work to consider the continum bandit problem.

KW algorithm was introduced in~\cite{kiefer1952stochastic}, and since then its convergence rate, and the asymptotic distribution of the estimates have been established~\cite{derman1956application,fabian1967stochastic,polyak1990optimal}. However, we note that in general the asymptotic convergence rate of an algorithm does not imply regret bounds. 

~\cite{cope2009regret} performs a regret analysis for the KW algorithm when the function is kept constant. In contrast with the work in~\cite{cope2009regret} we consider the non-stationary set-up in which the distribution of the reward sequence, or the unknown function to be maximized, changes over time. A regret analysis in this case amounts to controlling the performance of the algorithm over \emph{all possible} sequences of functions $\{f_s\}_{s=1}^{T},f_s\in\mathcal{C}$.

We analyze two popular variants of the KW algorithm for the context of non-stationary function maximization i) KW with constant step-size $\beta$, where $\beta$ is the ``learning rate" and ii) KW with sliding window of length $L$, also denoted ``memory length". We impose restrictions on the class $\mathcal{C}$ of allowable functions, and obtain bounds on the regret of $KW_{\beta}$, $KW_{L}$ algorithms in terms of the degree of non-stationarity, i.e. the quantity $\frac{\Delta_T}{T}$, where $\Delta_T$ is the number of times that the function $f_s$ being sampled changes until time $T$. 

We obtain the optimal learning rate $\beta^\star$, and window length $L^{\star}$ in terms of $\frac{\Delta_T}{T}$. We then show that if these KW variants use optimal $\beta^\star$ (resp. $L^\star$), then they are asymptotically efficient, i.e., their cumulative regret is asymptotically $0$ if $\lim_{T\to\infty}\frac{\Delta_T}{T}=0$.
\section{Non-stationary Function Maximization and Regret}
At each time $s=1,2,\ldots$, an allocation rule $\mathcal{A}$ chooses the \emph{control action} $X_s\in \mathcal{D}\subset\mathbb{R}^d$. We assume that $\mathcal{D}$ is convex and compact. The (random) reward earned at time $s$ is then equal to $F_s$.
 If $\mathcal{A}$ chooses the action $x$, then the distribution of the reward at time $s$, i.e., $F_s$, is given by $G(\cdot,x,f_s(x))$, and the mean value of the reward earned is $f_s(x)$, i.e., $\mathbb{E}\left\{F_s | X_s = x\right\} = f_s(x)$. We assume that the functions $f_1,f_2,\ldots$ belong to a function class $\mathcal{C}$, for each $f\in\mathcal{C}$, $f(x)$ is bounded for all $x\in\mathcal{D}$. Equivalently, the algorithm $\mathcal{A}$ obtains a ``noisy version" of the true function $f_s(\cdot)$ evaluated at $x$. At time $s$, Algorithm $\mathcal{A}$ observes its control action $X_s$ and the reward $F_s$, however it does not observe the function $f_s$. The control algorithm/allocation rule $\mathcal{A}$, for each time $s$, maps the history $\{X_n,F^{+/-}_n\}_{n=1}^{s-1}$ to an action $x\in \mathcal{D}$.

Denote by $f_{[1:s]}$ the sequence of functions $f_1,f_2,\ldots,f_s$, and for a function $f$, denote by $\theta(f)$ the value of $x$ that maximizes $f$. The total regret accumulated by an algorithm $\mathcal{A}$ until time step $T$ is then defined to be,
 \begin{align*}
 \mathcal{R}(T,\mathcal{A},f_{[1:T]}) = \mathbb{E} \left\{\sum_{s=1}^{T} f_s(\theta(f_s)) - f_s(X_s)   \right\},
 \end{align*} 
where expectation is taken with respect to the probability measure induced by the control algorithm $\mathcal{A}$ which makes the choice of the sampling sequence $\{X_s\}_{s=1}^{T}$ and the observations $\{F^{+/-}_s\}_{s=1}^{T}$. 

We will be interested in \emph{worst-case regret} of the algorithm $\mathcal{A}$, i.e., the quantity,
 \begin{align}\label{rp}
 \mathcal{R}(T,\mathcal{A}) = \sup_{f_{1:T}: f_s \in \mathcal{C}~\forall s\in\left[1,T\right]}\mathcal{R}(T,\mathcal{A},f_{[1:T]}).
 \end{align} 
The control algorithm $\mathcal{A}$ is asymptotically efficient~\cite{lai1985asymptotically} if  
 \begin{align}\label{rp1}
 \limsup_{T\to\infty}\frac{ \mathcal{R}(T,\mathcal{A})}{T} = 0.
 \end{align} 
%and for $\alpha \in \left[0,1\right]$ we say that the regret of $\mathcal{A}$ scales as $T^\alpha$ if
 %\begin{align}\label{rp2}
 %\limsup_{T\to\infty}\frac{ \mathcal{R}(T,\mathcal{A})}{T^{\alpha}} > 0.
 %\end{align} 
Next, we impose some restrictions on the allowable function class $\mathcal{C}$ that will enable us to obtain meaningful bounds on the regret. 
\section{Assumptions on the function class $\mathcal{C}$} 
We now make certain assumptions on the function class $\mathcal{C}$ from which the functions $f_s,s=1,2,\ldots$ are chosen. This allows us to obtain non-trivial bounds on the regret~\eqref{rp}. The conditions mentioned below are mostly taken from~\cite{cope2009regret}. \begin{condition}\label{condi1}
%This is Condition 1.1 in~\cite{cope2009regret}.
 Let $f\in \mathcal{C}$. Then $f$ is three times continuously differentiable for all $x\in \mathcal{D}$, and there exist positive constants $K_1,K_2$ such that the following hold for all $x\in\mathcal{D}$:
\begin{align}
-K_1 \|x-\theta(f)\|^2 &\geq (x-\theta(f))^\intercal\nabla f(x)\label{cond1}\\
\|\nabla f(x)\|&\leq K_2 \|x-\theta(f)\|.\label{cond2}
\end{align}
We refer to these conditions as Concavity-Like Condition (CL) and Linearly Bounded Growth Rate (LBG) respectively.
\end{condition}
%The next condition is Condition 3.1 in~\cite{cope2009regret}.
\begin{condition}\label{cond3}
There exists $K_3>0$ such that for all $f\in \mathcal{C}$ and $x\in\mathcal{D}$,
\begin{align*}
f(\theta)-f(x)\leq K_3\|x-\theta(f)\|^2.
\end{align*}
We refer to this condition as Quadratically Bounded (QB) function.
\end{condition} 
Other than the various ``smoothness" criteria that we assumed on the function $f$, we also need to ensure that the sampling noise is sufficiently well-behaved. We impose a uniform bound on the noise variance at each sample point, i.e., 
\begin{condition}\label{cond:3}
\begin{align}\label{adhoc3}
\int (y-f(x))^2 g(y;x,f(x))\mathrm{d}\mu(y)<\sigma^2, \forall f\in\mathcal{C}, x\in\mathcal{D},
\end{align}
where $\mu$ is a $\sigma$ finite measure on $\mathcal{B}(\mathbb{R})$, i.e., Borel sets of $\mathbb{R}$, and $g(\cdot;x,u)$ is the density of the random reward earned when the control action is $x$, and the mean value of reward is $f(x) = u$. %Note that the condition~\eqref{adhoc3} the bounds noise variance across all $f\in\mathcal{C}$. 
\end{condition}
Examples of function classes $\mathcal{C}$ which satisfy the above stated conditions can be found in~\cite{cope2009regret}. We now state the KW algorithm with fixed step-sizes, i.e., $\beta_s\equiv \beta$ and $c_s\equiv c$.
The following assumption on the function class $\mathcal{C}$ is in the spirit of the Mean Value Theorem.
\begin{condition}\label{cond4}
Let $M_s(X_s): =\mathbb{E}\left(\frac{F^{+}_s-F_s^{-}}{2c}\big|X_s=x,f_s=f \right)$. If the parameter $c$ is chosen to be sufficiently small, %then for all $x\in \mathcal{D}$ and for all $f\in\mathcal{C}$,
\begin{align*}
M_s(X_s) = \nabla f (X_s+ \epsilon_{X_s}), 
\end{align*}
where $\|\epsilon_{X_s}\| < \epsilon$, and moreover $\epsilon<c^2$.
\end{condition}

%We note that the above condition implies that the parameter $c$ chosen for calculating an estimate of the derivative using~\eqref{eq:1}, is sufficiently small.  
\section{Variants of Kiefer-Wolfowitz Algorithm for non-stationary Bandit optimization}
We describe two variants of the basic KW algorithm, that are used when the function $f$ of interest is time-varying. Throughout, for two functions $a(t),b(t)$ we denote  $a(t)=o(b(t))$ if $\limsup_{t\to\infty}\frac{a(t)}{b(t)}=0$. %while we denote $a(t)=O(b(t))$ if $\lim_{t\to\infty}\frac{a(t)}{b(t)}=1$.
\subsection{KW with fixed step-size $\beta$ ($KW_{\beta}$)}
The KW algorithm with fixed step size has been discussed in~\cite{kushner}. It keeps the step-sizes $\beta_s,c_s$ to be a constant instead of slowly decaying them to $0$. Since the parameter $\beta_s$ corresponds to the ``learning" rate, the proposed algorithm places lesser weights to past samples, and hence ``eventually forgets the past estimates".  The \emph{KW with fixed step-size} is stated as follows : Let $\beta$ and $c$ be ``small" positive constants. The estimate of the optimal point at time $s$ evolves as,
\begin{align}\label{update}
X^i_{s+1} = X^i_{s}  +\beta\left\{ \left(\frac{F^+_{s}-F^-_{s}}{2c}\right)\right\},
\end{align}
where $F^+_{s},F^-_{s}$ are the measurement values at $X_s +/- ce$, and the vector $e=(1,1,\ldots,1)$. Henceforth, we will assume that the parameter $c$ has been chosen to be sufficiently small so that the Condition~\ref{cond4} is satisfied.
\subsection{KW with Sliding Window of length $L$ ($KW_{L}$)}
In the second variant of the KW algorithm, we fix an integer $L>0$, which is called ``window length" or ``memory size". At each time $s$, the algorithm uses only the latest $L$ function measurements in order to choose the action $X_s$. This is called KW with sliding window of length $L$, denoted $KW_{L}$. In the below, $X_0\in\mathcal{D}$ has been chosen at time $s=0$. At each time $s=1,2,\ldots$, the $KW_{L}$ algorithm utilizes the estimates of derivatives at past $L$ sample values $\{X_n\}_{n=s-L}^{s-1}$, and chooses the action $X_s$ according to,
\begin{align}\label{skw}
X_s = X_0 + \sum_{n=1}^{\min\left\{L,s\right\}} \beta_n \frac{Y_{n+ s-L}}{c_n},
\end{align} 
where $\beta_n = n^{1/2}, c_n = n^{1/4}$, and $Y_n$ the estimate of the derivative at $X_n$ and is given by~\eqref{eq:1}. Thus, the algorithm behaves as if at each time $s$, the original KW algorithm~\eqref{eq:1}-\eqref{eq:2} restarts with an initial value of $X_0$, and the estimate of the maximizer gets updated $L$ times. %, before the choice of $X_s$ is made with the first time slot shifted to $n-L$. 
Since the sample values $F^{+/-}_s$ that have been obtained at time $s$ will not be utilized for generating actions $X_{\tilde{s}}, \tilde{s}>s+L$, the algorithm ``forgets" samples that are ``older" than $L$ time units. This finite memory property enables it to adapt to non-stationary function. 
\subsection{Trade-off in choosing learning rates $\beta,L$}
The step-size $\beta$ corresponds to the learning rate of $KW_{\beta}$ algorithm, while the window length $L$ corresponds to the ``memory" of $KW_{L}$ algorithm.
Due to the non-stationary of the function $f_s$, there is a fundamental trade-off involved in choosing these parameters. If we have $f_s\equiv f$, then choosing a large vale of $L$ leads to a better convergence of the iterates to $\theta(f)$. However, when the $f_s$ is time-varying, a large value of $L$ will introduce the dependence of the current estimate $X_s$ on the past values of $f_t,t<s$. Since $f_t$ may not be equal to $f_s$, $L$ must be chosen appropriately in order to achieve a trade-off between the twin objectives of achieving a low-regret, while simultaneoulsy adapting to the changing function $f_s$. 
%As will be shown in our analysis, these parameters must be chosen proportional to square root of the  
% ``degree of non-stationarity" $\frac{\Delta_T}{T}$, where $\Delta_T$ is the number of times that the function $f_s$ has changed during the time duration $T$.
% . Thus, if the function $f$ is ``varying fast" one would choose $L$ to be small, so that present estimates do not depend on measurements that are supposed to have been derived from a function $\tilde{f}\neq f$. 
\vspace{-.15cm}
\section{$KW_{\beta}$ Preliminary Results for Stationary case, $f_t\equiv f$}
%\emph{Recursive Equation for $\mathbb{E}\|X_n-\theta\|^2$ for a fixed $f$} :\\
In this section we present some results that will be used in later sections in order to perform a regret analysis of the two variants of KW algorithm that have been introduced. Throughout this section we will assume that the function $f$ that is being sampled is kept fixed, i.e., $f_s\equiv f$, and $\theta$ is the maximizer. We begin by imposing a couple of conditions that are specifically utilized for analyzing the $KW_{\beta}$ algorithm. 
 % Within this section, we will assume $\theta = 0$ without loss of generality.
\begin{condition}[Uniform locally Lipschitz]\label{cond5}
For $x,y\in\mathcal{D}$ satisfying $\|x-y\|\leq \epsilon$, we have
\begin{align}
\|\nabla f(x) -\nabla f(y)\| \leq K_4 \|x-y\|, \forall f\in\mathcal{C}.
\end{align}
\end{condition}
\begin{condition}[Condition on step-size $\beta$]\label{cond6}
The step size $\beta$ is chosen as $\beta=c^{2/1-\alpha}$ where $\alpha\in(0,1)$.
\end{condition}
Let us now write the update equation~\eqref{update} in more detail. We note that $\mathbb{E} \left(F_s|X_s=x\right) = f(x)$,
and moreover the distribution of $F_s$ conditioned on the action $X_s=x$ is denoted $G(\cdot;x,f_s(x))$ and thus the noise distribution depends both on value of sampled point $x$, and the value of function $f_s(x)$.
We denote the following,
\begin{align}
Y_s &= \frac{F^+_{s}-F^-_{s}}{2c},\\
M_s(X_s):& =\mathbb{E}\left(Y_s|\mathcal{F}_s\right) = \frac{1}{2c} d(X_s,c),\label{Ms}\\
Z_s & = Y_s - \mathbb{E}\left(Y_s|\mathcal{F}_{s-1}\right)\label{marting}
\end{align}
where $d(x,c)$ denotes the vector of differences evaluated at $x$ with a step-size of $c$. $Z_s$ is the noise in observation of derivative.
The recursion~\eqref{update} can thus equivalently be re-written as,
\begin{align}\label{algo}
X_{s+1} = X_s + \beta \left(M_s(X_s)+ Z_s\right).
\end{align}
% xx note that the value of $c$ is left-out somewhere while performing regret analysis xx
  
%Without loss of generality assume $\theta = 0$. % Also let $c<\epsilon$.
%\begin{lemma}[Bounding the derivative $M_s(x)$]\label{lemma}
%For a function $f\in\mathcal{C}$, we have that,
%\begin{align*}
 %\|M_n(x)\| \leq 2K_2 \|x\|.
 %\end{align*}
%\end{lemma}
%\begin{proof}
%Let us assume that $x$ is outside the $\epsilon$ ball centered at $0$. We have that,
%\begin{align*}
%\|M_s(x)\| = \|\nabla f(x+\vec{\epsilon}_x)\|&\leq K_2 \|x+\vec{\epsilon}_x\|\\
%&\leq K_2 \left(\|x\| + \|\vec{\epsilon}_x\|\right)\\
%&\leq K_2 \left(\|x\| + c\right)\\
%&\leq K_2 \left(\|x\| + \|x\|\right)\\
%&=2K_2 \|x\|,
%\end{align*}
%where first equality follows from Condition~\ref{cond4}, second inequality from condition~\eqref{cond2}, third is simply the triangle inequality, fourth inequality follows from Condition~\ref{cond4}, and the last from the assumption that $\|x\|$ is large enough.
%\end{proof}
From the recursion~\eqref{algo}, i.e., $X_{s+1} = \left(X_s + \beta M_s(X_s)\right)+\beta  Z_s$ we have that,
\begin{align}\label{eq:3}
&\|X_{s+1}-\theta\|^2 = \|X_s-\theta\|^2 + \beta^2 \|M_s(X_s)\|^2\notag\\
&+2\beta \left(X_s-\theta\right) M_s(X_s)^{\intercal} +\beta^2 \|Z_s\|^2 \notag\\
&+2\beta Z_s \left(X_s -\theta+ \beta M_s(X_s)\right)^{\intercal}.
\end{align}
Next, we use the conditions imposed on $\mathcal{C}$ and obtain a simple-to-analyze recursion for analyzing the quantity $\mathbb{E}\|X_s-\theta\|^2$.
\begin{lemma}\label{lemma2}
If the Conditions 1,3 and 4 hold true, then for the recursions~\eqref{algo}, we have that,
\begin{align}\label{ineq1}
\mathbb{E}\left\{\|X_{s+1}-\theta\|^2|\mathcal{F}_s \right\} \leq \gamma \|X_{s}-\theta\|^2+ H(\beta),
\end{align}
where
\begin{align}
\gamma&: = 1-2\beta K_1+2\beta^2 K^2_2<1,\mbox{ and }\label{eq:gamma}\\
 H(\beta)& :=\frac{\beta^2\sigma^2}{c^2}+2KK_4 \beta\epsilon+ 2\beta^2K_2^2\epsilon^2\label{eq:k5},
\end{align} 
where $K$ is the diameter of the set $\mathcal{D}$, step-size $\beta$ is chosen to be sufficiently small in order that $\gamma<1$, and $\tilde{\sigma}^2:=4d\sigma^2$
\end{lemma}
%(xx check this xx)
\begin{proof}
The term $\beta^2\|M_s(X_s)\|^2$ can be bounded as follows
\begin{align}\label{adhoc1}
\beta^2 \|M_s(X_s)\|^2 &= \beta^2 \|\nabla f(X_s+\epsilon_{X_s})\|^2\notag\\
&\leq \beta^2K_2^2 \|X_s+\epsilon_{X_s}-\theta\|^2\notag\\
&\leq \beta^2K_2^2 \left(\|X_s-\theta\|+\epsilon\right)^2\notag\\
&\leq 2\beta^2 K_2^2 \left(\|X_s-\theta\|^2 + \epsilon^2\right),
\end{align}
where the first equality follows from Condition~\ref{cond4}, while the first inequality follows from the inequality~\eqref{cond2} of Condition~\ref{condi1}, while second inequality follows from the triangle inequality, and the last inequality follows since for $x,y\in\mathbb{R}$, we have $(x+y)^2\leq 2(x^2+y^2)$.

Next, we have  %$X_s \nabla f(X_s)\leq -K_1 \|X_s\|^2$
%in the relation~\eqref{eq:3}
\begin{align}\label{adhoc2}
&\left(X_s-\theta\right)M_s(X_s)^{\intercal} = \left(X_s-\theta\right)\nabla f(X_s+\epsilon_{X_s})^{\intercal}\notag\\
& = \left(X_s-\theta\right)\left(\nabla f(X_s)+\nabla f(X_s+\epsilon_{X_s})-\nabla f(X_s)\right)^{\intercal}\notag\\
& = \left(X_s-\theta\right)\nabla f(X_s)^{\intercal}\\
&~+ \left(X_s-\theta\right)\left(\nabla f(X_s+\epsilon_{X_s})-\nabla f(X_s)\right)^{\intercal}\notag\\
&\leq -K_1\| X_s-\theta\|^2 + K K_4\epsilon,
\end{align}
where the first equality follows from Condition~\ref{cond4}. For the last inequality, the bound on the first term follows from~\eqref{cond1}, while that on the second term follows from Cauchy-Schwartz inequality used in conjunction with Condition~\ref{cond5}. 
 
%Utilizing the inequalities~\eqref{adhoc1} and~\eqref{adhoc2} we obtain,
%\begin{align}\label{ineq2}
%&\|X_s-\theta\|^2 + \beta^2 M_s(X_s)^2+2\left(X_s-\theta\right) \beta M_s(X_s)\notag\\
% &\leq \|X_s-\theta\|^2 + 2\beta^2 K^2_2\left(\|X_s-\theta\|^2+c^2\right)-2\beta K_1 \|X_s\|^2\notag\\
%&+2\beta K_2\|X_s-\theta\|\epsilon\notag\\ 
%&= \|X_s-\theta\|^2 \left\{ 1+2\beta^2 K^2_2-2\beta K_1\right\}+2\beta K_2\|X_s-\theta\|\epsilon +2\beta^2K_2^2c^2 \notag\\
%&\leq \gamma \|X_s\|^2 + 2\beta KK_2 \epsilon +2\beta^2K_2^2c^2 ,
%\end{align}
%where $\gamma: = 1+2\beta^2 K^2_2-2\beta K_1$, and the last inequality uses $\|X_s-\theta\|\leq K$

Next, it follows from~\eqref{marting} that expectation of $2\beta Z_s \left(X_s-\theta + \beta M_s(X_s)\right)^{\intercal}$ conditioned on $\mathcal{F}_{s-1}$ is $0$.
Also, from Condition~\ref{cond:3} we have that $\frac{\beta^2}{c^2}\mathbb{E}\left( \|Z_s\|^2|\mathcal{F}_{s-1}\right)= \beta^2\mathbb{E}\left( \sum_{i=1}^{d}(F^{+}_s(i)-F^{-}_s(i))^2|\mathcal{F}_{s-1}\right)\leq \frac{\beta^24d\sigma^2}{c^2}$, since the random variable $Z_s$ conditioned on the filtration $\mathcal{F}_{s-1}$ is the value of noise in the current estimate of the function gradient, and we imposed a uniform bound on the variance of this noise. This yields us %xx
\begin{align}\label{ineq3}
\mathbb{E}\left(\beta^2 \|Z_s\|^2 +2\beta Z_s \left(X_s-\theta + \beta M_s(X_s)\right)^{\intercal} |\mathcal{F}_{s}\right)\leq \frac{\beta^2\tilde{\sigma}^2}{c^2}.
\end{align}
The proof is now completed by substituting the inequalities~\eqref{adhoc1},~\eqref{adhoc2} and~\eqref{ineq3} in the expression~\eqref{eq:3} and letting
$\gamma = 1-2\beta K_1+2\beta^2K_2^2$ and $H(\beta)$ as in~\eqref{eq:k5}.
\end{proof}
\subsection{Regret Analysis with fixed $f$}
Taking unconditional expectation in the expression~\eqref{ineq1}, and solving for the ensuing recursions we obtain,
\begin{align}\label{closedform1}
\mathbb{E}\|X_s-\theta\|^2 \leq H(\beta) \frac{(1-\gamma^s)}{(1-\gamma)} + \|x_0-\theta\|^2 \gamma^s.
\end{align}
It follows from Condition~\ref{cond3} that the regret at time $s$, i.e., the quantity  $f(\theta)-f(X_s)$ can be bounded in terms of the distance $\|X_s-\theta\|^2$, 
\begin{align}\label{closedform2}
\mathbb{E}f(\theta)-f(X_s)\leq K_3\left(H(\beta) \frac{(1-\gamma^s)}{(1-\gamma)} + \|x_0-\theta\|^2 \gamma^s\right).
\end{align}

Thus, we see that the instantaneous regret at time $s$ or equivalently the ``distance" of the current estimate $X_s$ from the optimal point $\theta$ can be decomposed into the following two components:
\begin{enumerate}
\item Regret due to incomplete learning: i.e., the quantity $K_3\|x_0-\theta\|^2 \gamma^s$ which is the error between the current estimate $X_s$ and the true maximizer $\theta$. Note that for a fixed value of $\gamma$, this component decreases with increasing $s$, so that the $KW_{\beta}$ algorithm improves upon the estimate of $\theta$ as it obtains more information about the function $f$ with time.
 %the ``settling error" 
\item Regret due to Noisy Estimate of $\nabla f$:  $K_3H(\beta) \frac{(1-\gamma^s)}{(1-\gamma)}$ resulting from noisy measurements of the gradients $\nabla f(x)$. Note that if the step-size $\beta$ was allowed to decay as in~\eqref{eq:2}, then the noise would ``average-out" and its limiting contribution will be $0$ almost surely.
\end{enumerate}
The regret decompositon~\eqref{closedform1} throws light on the fundamental trade-off presented in the non-stationary setting. The contribution of 2) is increasing in the learning-rate $\beta$.
Indeed, if the function were stationary, i.e., $f_t \equiv f$, one could asymptotically ``stop-learning" by letting $\beta_t \to 0$ asymptotically, so that 2) would vanish. Due to non-stationarity, $\beta$ has to be kept constant at a ``small value". However, for small values of $\beta$, from~\eqref{eq:gamma} we have $\gamma \approx 1-2K_1\beta$, so that a small $\beta$ implies a larger learning regret, i.e., the algorithm takes a long time to learn the function maxima. Thus, the ``optimal" choice of $\beta$ amounts to obtaining an optimal trade-off between the components 1) and 2) of the instantaneous regret. 

 We will now evaluate the expressions for each of these regret terms.
 \begin{lemma} [ Regret under fixed $f\in\mathcal{C}$]\label{lemma:regret}
Consider the allocation rule~\eqref{update}, i.e, KW with constant step-size $\beta$, applied to find the maximizer of an unknown function $f\in\mathcal{C}$. Let the time-horizon be fixed at $T$, and the function class $\mathcal{C}$ and step-size $\beta$ satisfy Conditions 1-6. The cumulative regret incurred during the period $\{1,2,\ldots,T\}$ can be upper-bounded as
\begin{align}\label{upperbound1}
\mathbb{E}\left(\sum_{s=1}^{T}\|f(X_s)-f(\theta)\|\right) \leq  \frac{K_3H(\beta)T}{1-\gamma} + \|X_0-\theta\|^2 \frac{K_3}{1-\gamma}.
\end{align}
Consider the learning rate
\vspace{-.15cm}
\begin{align}\label{bstar}
\beta^\star =\Lambda/T^{1/(2+\alpha)},
\end{align}
where $\Lambda = \left(\frac{K^2}{\tilde{\sigma}^2}\right)^{1/(2+\alpha)}$ is a constant that depends upon the function class $\mathcal{C}$, and $\tilde{\sigma}^2=4d\sigma^2$. The regret incurred by $KW_{\beta^\star}$ is then upper-bounded as
\begin{align}\label{regret}
&\frac{\mathcal{R}(T,KW_{\beta^\star})}{T}\leq \frac{K_3}{2K_1}\left( \Lambda^\alpha T^{-1/(2+\alpha)} +2KK_4\Lambda^\alpha T^{-1/(2+\alpha)}\right.\notag\\
&\left.+2\Lambda^3T^{-3/(2+\alpha)} +\frac{K^2}{\Lambda}T^{-(1+\alpha)/(2+\alpha)}\right),
\end{align}
and hence we have that
\begin{align}
\limsup_{T\to\infty}\frac{\mathcal{R}(T,KW_{\beta^\star})}{T} =0.
\end{align} 
 \end{lemma}
 %If the time-horizon $T$ is known to KW-$\beta$, then  
 \begin{proof}
We note that since from Condition~\ref{cond3} we have that for each $f\in\mathcal{C}$ the regret $f(\theta)-f(X_s)$ can be bounded within a factor of $K_3$ from $\|X_s-\theta\|^2$, rest of the discussion will be focused on bounding the latter term, and we will occasionally call it ``regret", or ``estimation error".

The instantaneous regret at time $s$ is bounded as in~\eqref{closedform1}.
The contribution of the term $H(\beta) \frac{(1-\gamma^s)}{(1-\gamma)}$ 
is upper-bounded by $H(\beta) \frac{1}{(1-\gamma)}$, so that the cumulative regret due to the first term of~\eqref{closedform1} is bounded by $H(\beta) \frac{T}{(1-\gamma)}$.
Also, $\sum_{s=0}^{T}\|x_0-\theta\|^2 \gamma^s =\|x_0-\theta\|^2 \frac{1-\gamma^{T}}{1-\gamma}\leq K^2\frac{1}{1-\gamma}$, where $K$ is the diameter of the set $\mathcal{D}$.
This yields us the bound~\eqref{upperbound1}. The proof of regret bound~\eqref{regret} follows by substituting the value of $\beta^\star$ from~\eqref{bstar}, and $\gamma,H(\beta)$ from \eqref{eq:gamma},\eqref{eq:k5} into the bound \eqref{upperbound1} and performing simple algebraic manipulations.
 %Next, we restrict ourselves to small values of $\beta$, so that $H(\beta) = \frac{\beta^2}{}$ 
 %Finding the optimal $\beta$ that minimizes the bound~\eqref{upperbound1} reduces to solving the following optimization problem,
%\begin{align*}
%\min_{\beta\in \left[0,1\right]} \left(\beta \frac{\sigma^2}{c^2} + 2KK_4\beta \right) T + \frac{K^2}{\beta}.
%\end{align*}
%For $\beta$ sufficiently small, we use~\eqref{eq:gamma}, and let $\gamma = 1-2\beta K_1 $, so that the above-stated problem reduces to 
%\begin{align*}
%\min_{\beta\in \left[0,1\right]} \frac{\beta \sigma^2}{2K_1c^2} T + \frac{K^2}{2K_1\beta }, 
%\end{align*}
%which yields $\beta^\star = \frac{Kc}{\sigma}\frac{1}{\sqrt{T}}$. 
\end{proof} 
%The relation~\eqref{optbeta} provides insight on how to choose the learning rate. $\beta^\star$ obtains the optimal trade-off by balancing the ``learning regret" with ``noise-cancellation". 
%\begin{align}\label{optbeta}
%\beta^{\star} =\frac{Kc}{\sigma}\frac{1}{\sqrt{T}},
%\end{align}
%where $K$ is the diameter of the compact set $\mathcal{D}$. 
%The regret incurred by KW-$\beta^\star$ is then upper-bounded as
%\begin{align}
%\mathcal{R}(T,KW(\beta^\star))\leq \frac{KK_3\sigma}{K_1 c} \sqrt{T}.
%\end{align}
\section{Regret Analysis of $KW_{\beta}$ for Non-Stationary case }
We begin by introducing some notation. Since the function $f_s$ changes with time, let us denote by $\tau_1,\tau_2,\ldots$ the times at which the functions change. We will denote the set $\{x,x+1,\ldots,y\}$ by $[x,y]$. Thus, for each of the individual ``episodes" comprising of time intervals
$[0,\tau_1],[\tau_1+1,\tau_2],[\tau_2+1,\tau_3],\ldots$, we have that $f_{\tau_i}=f_{\tau_i + 1}=\cdots=f_{\tau_{i+1}-1}$. Also denote by $\Delta_T$ the number of episodes until time $T$.
 For a function $f\in\mathcal{C}$, let $\theta(f)$ be the value of $x$ that maximizes the function $f$.
 Let $\theta_s$ denote the maxima of the function $f_s$. Thus, if $s\in [\tau_i+1,\tau_{i+1}]$, then $\theta_s=\theta_{\tau_i}=\theta(f_{\tau_i})$. 
%Let us now assume that $\theta$ is varying with time, and denote by $\theta_n$ its value at time $n$, i.e., $\theta_n(\omega)$ is now a stochastic process that is adapted to the filtration $\mathcal{\tilde{F}}_n$. Also recall that the  control process $X_n$ is adapted to the filtration $\mathcal{F}_n$ which is assumed to be generated by the joint control-observation samples $\left(X_s,Y_s\right)_{s=1}^{n}$. We will assume that $\mathcal{\tilde{F}}_n\subsetneq\mathcal{F}_n$.
 %Thus, this captures the ``adversarial" setting too.
We will denote by $\theta_{[1:T]}$ the sequence $\theta_1,\theta_2,\ldots,\theta_T$, similarly for $f_{[1:T]}$.

Next, we will perform a sample-path performance analysis of $KW_{\beta}$ algorithm. Thus, fix a sequence $f_{[1:T]}$ with the corresponding $\theta_s$ sequence given by $\theta_{[1:T]} =\theta_1,\theta_2,\ldots,\theta_T$. Moreover, for each episode $i=1,2,\ldots,\Delta_T$ denote by 
$T_i: = \tau_{i+1}-\tau_i$, to be the ``episode-length" or horizon length of episode $i$. Since the cumulative regret incurred over the time horizon $T$ can be decomposed into the sum of regrets incurred during individual episodes composed of time intervals $\left\{\left[\tau_i,\tau_{i+1}-1\right]\right\}_{i=1}^{\Delta_T}$, the regret incurred by $KW_{\beta}$ is then equal to,
\begin{align}\label{tower}
&\mathbb{E}\sum_{s=1}^T f_s(\theta_s) - f(X_s) \notag\\
& = \sum_{i=1}^{\Delta_T} \mathbb{E}\mathbb{E}\left\{ \sum_{s=\tau_i}^{\tau_{i+1}-1}f_{\tau_i}(\theta_{\tau_i})-f_{\tau_i}(X_s) \bigg|\mathcal{F}_{\tau_i}\right\}, 
\end{align}
where $\mathcal{F}_s$ is the filtration generated by the random variables $\{(X_n,Y_n,F^{+/-}_n)\}_{n=1}^{s}$.
We now analyze the regrets incurred during the interval $\left[\tau_i,\tau_{i+1}-1\right]$.

We will work with the distance $\|X_s-\theta_s\|^2$ in lieu of $f_s(\theta_s)-f_s(X_s)$, with the understanding that the regret can be upperbounded within a constant factor of the former by using Condition~\ref{cond3}. Since during the episode $i$, the function $f$ being sampled, and its maximizer $\theta(f)$ are equal to $f_{\tau_i},\theta_{\tau_i}$ respectively, and the inequality~\eqref{closedform1} holds for \emph{all} $f\in\mathcal{C}$, we can use the bound~\eqref{closedform1}. Thus, the regret incurred during the $i$-th episode can be bounded by utilizing the bound~\eqref{upperbound1} developed in Lemma~\ref{lemma:regret}. However, the term $X_0$ will be replaced by the quantity $X_{\tau_i}-\theta_i$ to account for the difference between the estimate $X_{\tau_i}$ at beginning of episode $i$, and the true maximizer $\tau_{\tau_i}$ during episode $i$. Similarly, the horizon $T$ will be replaced by the episode length $T_i$. This yields us,  
\begin{align}\label{eq:4}
\mathbb{E}\left\{ \sum_{s=\tau_i}^{\tau_{i+1}-1}\|X_s-\theta_{\tau_i}\|^2 \bigg|\mathcal{F}_{\tau_i}\right\} &< H(\beta) \frac{T_i}{1-\gamma} + \frac{\|X_{\tau_i}-\theta_{\tau_i}\|^2}{1-\gamma},\notag\\
&\leq H(\beta)\frac{T_i}{1-\gamma} + \frac{K^2}{1-\gamma},
\end{align}
where the second inequality follows since we can bound the distance $\|X_{\tau_i}-\theta_{\tau_i}\|$ by the diameter of the set $\mathcal{D}$, i.e., $K$. Combining the above bound with the tower property of conditional expectations~\eqref{tower}, we obtain the following result.
\begin{theorem}
Consider the problem of designing optimal allocation rule for the non-stationary set-up, and for each time $s=1,2,\ldots,T$, let the function $f_s\in\mathcal{C}$. Let the function class $\mathcal{C}$ satisfy the conditions 1-5. 

The regret incurred by $KW_{\beta}$ algorithm during the time horizon $T$ is upper-bounded by,
\begin{align}\label{kw:ub}
 \mathcal{R}(T,KW_{\beta}) \leq \frac{H(\beta)K_3T}{(1-\gamma)} + \frac{K^2\Delta_TK_3}{1-\gamma},
\end{align}
so that with the learning rate $\beta$ set equal to
\begin{align}
\beta^\star =\Lambda\left(\frac{\Delta_T}{T}\right)^{1/(2+\alpha)},
\end{align}
where $\Lambda = \left(\frac{K^2}{\tilde{\sigma}^2}\right)^{1/(2+\alpha)}$,
we have that 
\begin{align}
\small
&\left(\frac{\mathcal{R}(T,KW_{\beta^\star})}{T}\right)\frac{2K_1}{K_3}\leq  \Lambda^\alpha \left(\frac{T}{\Delta_T}\right)^{\frac{-1}{(2+\alpha)}} \notag\\
&+2KK_4\Lambda^\alpha \left(\frac{T}{\Delta_T}\right)^{\frac{-1}{(2+\alpha)}}\notag\\
&+2\Lambda^{3}\left(\frac{T}{\Delta_T}\right)^{\frac{-3}{(2+\alpha)}}+\frac{K^2}{\Lambda}\left(\frac{T}{\Delta_T}\right)^{\frac{-(1+\alpha)}{(2+\alpha)}},
\end{align}
so that if $\Delta_T=o(T)$, we have
\begin{align}
\limsup_{T\to\infty}\frac{\mathcal{R}(T,KW_{\beta^\star})}{T} =0,
\end{align} 
%which, for sufficiently small $\beta$ reduces to,
%\begin{align}\label{kw:ub1}
%\mathcal{R}(T,KW(\beta)) \leq \frac{\beta\sigma^2K_3}{c^2} \frac{T}{2K_1} + \frac{K^2\Delta_TK_3}{2K_1 \beta}.
%\end{align}
%The value of $\beta$ that minimizes the regret bound is,
%\begin{align}\label{betaopt}
%\beta^\star = \frac{Kc}{\sigma} \sqrt{\frac{\Delta_T}{T}},
%\end{align}
%and the regret incurred by KW-$\beta^\star$ is bounded as,
%\begin{align}
%\mathcal{R}(T,KW(\beta^\star))\leq \frac{KK_3\sigma}{K_1 c}\sqrt{T\Delta_T},
%\end{align}
%Thus, if $\Delta_T = o\left(T\right)$, then under the KW-$\beta^\star$ sampling control rule, the asymtotic regret,
%\begin{align*}
%\limsup_{T\to\infty}\frac{\mathcal{R}(T,KW(\beta))}{T} = 0.
%\end{align*}
\end{theorem} 
%Next, we perfom a regret analysis of the sliding window KW for the case of non stationary $\theta$.
\section{Regret Analysis of KW with Sliding Window}
We begin with the case where the function is held fixed at $f_s\equiv f$, and time-horizon is fixed at $T$. Let $L$ denote the length of window, and $\theta$ be the maximizer of $f$. Next, we can apply Chung's Lemma as in Lemma III.5 of~\cite{cope2009regret}, in order to analyze the asymptotic properties of the distance $\|X_s-\theta\|$.
\begin{lemma}\label{coro1}
Let the function class $\mathcal{C}$ satisfy the Conditions~\ref{condi1}-\ref{cond:3}.
For the KW with sliding window of length $L$ applied to obtain the maxima of a stationary function $f\in\mathcal{C}$, the following is true. There exists an integer $s_0>0$ such that,
\begin{align*}
\mathbb{E} \|X_s-\theta\|^2 \leq \frac{K_5}{\sqrt{L}}, \forall s>\max \left\{s_0, L\right\},
\end{align*}
where the constants $K_5$ and $s_0$ depend on the function class $\mathcal{C}$ only through the values $K_1,K_2,K_3$.
\end{lemma}

Throughout, we will assume that the window length $L$ has been chosen so that it satisfies $L>s_0$, and hence the bound above can be written as 
\begin{align}\label{swb}
\mathbb{E} \|X_s-\theta\|^2 \leq K_5/\sqrt{L}, \forall s>L.
\end{align}
Next, we consider the non-stationary set-up. Fix a sequence $f_{[1:T]}$, and the corresponding $\theta_{[1:T]}$, and as before let $\Delta_T$ be the number of episodes until time $T$. Let us analyze the regret incurred during the $i$-th episode that is of duration $T_i = \tau_{i+1}-\tau_i$. Since the control action $X_s$ generated at times $s\in \left[\tau_i+1,\tau_{i+1}\right]$ is a function of the values $\{Y_n\}_{n=s-L}^{s-1}$, the regret bound~\eqref{swb} which was derived for stationary set-up can now be applied only when $s-\left(\tau_{i}+1\right)>L$ or equivalently $s>L+ \tau_{i}+1$. This gives us the following,
\begin{lemma}
Let $KW_{L}$ algorithm be applied to the non-stationary function maximization problem. Consider the process $X_s-\theta$ during the episode $i$, which is comprised of time interval $ \left[\tau_i+1,\tau_{i+1}\right]$. If the episode length $\tau_{i+1}-\tau_i$ is greater than $L$, then, we have
\begin{align}\label{swb1}
\mathbb{E}\left\{ \|X_s-\theta_{\tau_i}\|^2 | \mathcal{F}_{\tau_i} \right\}\leq \frac{K_5}{\sqrt{L}}, \forall s\in \left[\tau_i+L,\tau_{i+1}\right].
\end{align}
Thus, the total regret incurred during $i$-th episode can be bounded as follows,
\begin{align}\label{swb2}
&\mathbb{E} \left\{\sum_{s\in \left[\tau_i+1,\tau_{i+1}\right]}f_{\tau_i}(\theta_{\tau_i})-f_{\tau_i}(X_s) \bigg| \mathcal{F}_{\tau_i} \right\}\notag\\
&\qquad \leq K_3\left( \frac{K_5\left(T_i-L\right)^{+}}{\sqrt{L}} + \left(L\wedge T_i\right)K\right)\notag\\
&\qquad \leq \frac{K_3K_5T_i}{\sqrt{L}} + LK_3K,
\end{align}
where $K$ is the diamater of the set $\mathcal{D}$, and the function $x^{+} = \max \left\{x,0\right\}$, and for $x,r\in\mathbb{R}$, the function $x\wedge y=\min(x,y)$.
\end{lemma}

\begin{theorem}
For the non-stationary bandit problem, the regret incurred by the $KW_{L}$ algorithm during time period $T$ can be bounded as,
\begin{align}\label{swb3}
\mathcal{R}\left(T,KW_L\right) \leq \frac{K_3K_5T}{\sqrt{L}} + LK_3K\Delta_T.
\end{align}
The choice of $L$ that minimizes the upper-bound is given by,
\begin{align}\label{lopt}
L^\star = \left( \frac{K_5}{2K} \frac{T}{\Delta_T}\right)^{2/3},
\end{align}
so that the regret under $KW(L^\star)$ is bounded as,
\begin{align}\label{swb5}
\frac{\mathcal{R}\left(T,KW(L^\star)\right)}{T} \leq K_5^{2/3}K^{1/3}\left(\frac{\Delta_T}{T}\right)^{1/3}\left[ 2^{1/3}+\frac{1}{2^{2/3}}  \right],
\end{align}
Thus, if the number of episodes $\Delta_T = o(T)$, then we have,
\begin{align}
\limsup_{T\to\infty}\frac{\mathcal{R}(T,KW(L^\star))}{T} = 0.
\end{align}
\end{theorem}
\begin{proof}
The bound~\eqref{swb3} is obtained by utilizing the upper-bound~\eqref{swb2} on the regrets incurred during individual episodes, in conjunction with the tower property~\eqref{tower} of conditional expectations. Rest of the proof involves simple algebraic manipulations, and is omitted due to space constraints. 
\end{proof}
%\section{Conclusion and Future Work}\label{sec:concl}

 \bibliographystyle{IEEEtran}
\bibliography{../../GTS/combinedbib}

\end{document}